\documentclass[sigconf,nonacm]{acmart}
\setcopyright{none}
\renewcommand\footnotetextcopyrightpermission[1]{}
\usepackage{balance}
\usepackage{cleveref}
\usepackage{listings}
\usepackage{enumitem}
\usepackage{algorithm}
\usepackage[noend]{algpseudocode}
\usepackage{multirow}     % for \multirow
\usepackage{array}        % extended table column definitions
\usepackage{rotating}     % for \rotatebox
\usepackage{graphicx}
\usepackage[dvipsnames]{xcolor}
\usepackage{amsthm}
\usepackage{needspace}
\definecolor{darkgreen}{RGB}{0,50,0}
\definecolor{darkorange}{RGB}{235,140,0}

\AtBeginDocument{%
  }

\begin{document}

\title{Generating Intervention Hypotheses using Explainable Explanations on Graphs: G2I, a Two-Stage Greedy Framework}

\author{Mulin Tian}
\affiliation{%
  \institution{University of Southern California}
  \city{Los Angeles}
  \state{California}
  \country{USA}
}
\email{mulintia@usc.edu}

\author{Ajitesh Srivastava}
\affiliation{%
  \institution{Northeastern University}
  \city{Charlotte}
  \state{North Carolina}
  \country{USA}
}
\email{aji.srivastava@northeastern.edu}

%% Short author list for page headers.
\renewcommand{\shortauthors}{Mulin Tian and Ajitesh Srivastava}

%% Optional: uncomment to record the accepted venue on the preprint.
%% (The published version is CC BY, so this is permitted.)
% \thanks{Accepted at the 35th ACM International Conference on Information
% and Knowledge Management (CIKM '26), November 7--11, 2026, Rome, Italy.
% \url{https://doi.org/10.1145/3799682.3840645}}

\begin{abstract}
Real-world decision-making in public health and social science can greatly benefit from predictive models, yet translating predictions into effective interventions requires explaining the model behavior. While Graph Neural Networks (GNNs) are well-suited for modeling relational data, existing explanation methods largely operate at the node level and fall short of supporting actionable, network-level intervention design. Existing counterfactual GNN explainers, such as CF-GNNExplainer and CF$^2$, rely on continuous mask optimization over features and edges, which implicitly assume feasible edge manipulation, may allocate effort to immutable or non-actionable attributes, and incur substantial computational overhead. Further, the method of arriving at the explanation itself is difficult to explain to a domain specialist who is not an AI expert. Can simple methods generate good explanations? To explore this, we reframe counterfactual explanation as an intervention design problem. At the local level, we generate counterfactuals via a greedy search that directly identifies minimal, actionable changes to node features and neighbor-level conditions. We derive 
conditions under which the greedy search provides guarantees, and empirically show that these conditions are approximately met.
These counterfactuals are converted into interpretable rules suitable for real-world intervention. At the network level, we formulate intervention selection as a Disjunctive Normal Form (DNF) coverage problem under a budget constraint, which is nondecreasing and approximately submodular, enabling a greedy algorithm with theoretical guarantees. Experiments on synthetic graphs and real-world suicide risk networks demonstrate that our approach produces scalable, cost-effective intervention strategies with significantly improved efficiency over mask-based counterfactual methods.

% \keywords{First keyword  \and Second keyword \and Another keyword.}
\end{abstract}

\keywords{Graph Neural Networks, Explainability, Intervention Generation.}

\maketitle

\section{Introduction}

%Suicide remains a pressing public health concern in underserved and high-risk populations. A wide range of machine learning models have been applied to identify individuals at risk, leveraging increasingly rich behavioral and social data. However, a key limitation of many predictive models lies in the interpretability and actionability of their outputs. In suicide prevention, accurately identifying at-risk individuals is only the first step; effective intervention requires understanding \emph{why} an individual is at risk and \emph{how} that risk can be mitigated.
In social and health sciences, interventions are performed to improve individual outcomes. These include identifying who could benefit from what kind of intervention.
For instance, to reduce the risk of suicide, the intervention may be designed to target those who could be at a higher risk based on certain historical, personal, or social attributes~\cite{a2019_christakis,jackson_2008_social}. 
Yet, identifying which individuals are at risk, and critically, why they are at risk, remains a significant challenge. Graph Neural Networks (GNNs)~\cite{wu2021comprehensive} have emerged as powerful tools for modeling social networks, as they can effectively encode both node-level attributes and the relational structure between individuals. This capability makes GNNs particularly suitable for predicting processes, where an individual's outcomes are shaped not only by personal characteristics but also by their social connections. However, beyond good accuracy, an explanation is needed to understand the model behavior.

Explanations as hypotheses: Several prior works, including GNNExplainer~\cite{ying2019gnnexplainer}, CF-GNNExplainer~\cite{cfgnnexplainer}, CF$^2$-GNNExplainer~\cite{cffgnn_explainer}, PGExplainer~\cite{PGExplainer}, SubgraphX~\cite{SubgraphX}, and RCExplainer~\cite{RCExplainer}, explain node-level predictions by identifying the node features and edges that are most influential to a target node. In particular, they can provide counterfactual explanations~\cite{wachter2018counterfactualexplanationsopeningblack,cfalgorithm} -- what minimal change in data can lead to a change in classification. These counterfactual explanations can be seen as `intervention hypotheses', i.e., if we are able to change certain attributes, it will change outcomes. Recent work in Economics has suggested the idea of using machine learning counterfactual explanations as a tool for hypothesis generation.~\cite{ludwig2024machine}. However, realizing this in real-world relational data for interventions is challenging due to the following reasons. (1) Existing approaches  provide node-level explanations and it is not clear how to convert it into population-level intervention. (2) Existing methods may not directly support explanations that lead to actionable and cost-effective prevention strategies. For example, an explanation for a node may suggest deleting an edge, which may not be possible or ethical in real networks, and does not provide population-level insights. (3) Most importantly, \textit{the explanation methods themselves are difficult to explain} to a non-ML expert. We argue that in cross disciplinary collaborations, the method for generating explanation of a ``black-box'' model should not itself be a black-box and instead should be easy to explain to a non-ML expert.
%In particular, mask-based explanations identify what is important for a prediction but fail to produce intervention strategies that are truly %actionable and cost-effective. 
%In real-world prevention settings, both node-level attributes and social ties are subject to practical, ethical, and logistical constraints, which these explanations do not explicitly account for, limiting their usefulness for designing feasible interventions. 

To address these limitations, we propose to use forms of greedy algorithms at two levels (G2I -- Graph to Intervention). At the ``local''-level, we perturb both structural and feature components of a network, greedily building a subset of features (for individual and their neighborhood) to maximize the shift in outcomes for node classification. At ``global''-level, we construct an intervention-design problem. Given the local counterfactual explanations, how to design an intervention targeting a subset of features (individual and neighborhood-level) to maximize improved outcomes. 
Our G2I framework consistently outperforms prior explanation methods, including CF~\cite{cfgnnexplainer} and CF$^2$~\cite{cffgnn_explainer}, in both the first-stage explanation evaluation and the second-stage intervention generation tasks. Moreover, G2I achieves up to two-orders of magnitude speedup compared to these methods. We summarize the novelty and benefits of our method compared to prior works below.

\noindent
\textbf{Explainability with Actionability:}
Our framework provides counterfactual explanations that explicitly characterize how changes to node attributes and social connections would alter a node’s predicted risk, moving beyond static importance scores toward explanations that reflect causal and contrastive reasoning.
By restricting feasible perturbations to features and neighborhoods that align with real-world intervention constraints, we enable practitioners to understand and design and test interventions.

\noindent
\textbf{Generalizable:}
Due to the greedy approach, our method does not require dataset-specific hyperparameter tuning and maintains stable performance across different datasets and structural settings.

\noindent
\textbf{Scalability:}
Compared to traditional mask-based optimization methods, our approach scales to significantly larger graphs under the same hardware constraints, as it only requires storing the original graph and a pre-trained GNN model, while prior methods require substantial memory overhead by maintaining mask variables of comparable size. This advantage is critical for population-level intervention design on large-scale social networks.

\section{Background and Related Work}

\noindent\textbf{{GNN Explanation Methods.}}
GNNExplainer\cite{ying2019gnnexplainer} learns a soft mask over edges and node features to identify a subgraph that is sufficient to preserve the original GNN prediction. However, it does not model necessity, as it does not test whether removing the identified components would change the outcome. Consequently, the explanations are descriptive but not counterfactual or actionable. CF-GNNExplainer\cite{cfgnnexplainer} generates counterfactual explanations by identifying minimal edge deletions that flip a node’s prediction, explicitly capturing necessity. However, it does not enforce sufficiency, since the remaining subgraph is not required to reproduce the original prediction. The method is further limited to structural perturbations and ignores feature–structure interactions. \textbf{CF$^2$}\cite{cffgnn_explainer} explicitly incorporates both sufficiency and necessity by jointly optimizing factual and counterfactual objectives. It ensures that the selected subgraph is sufficient to maintain the prediction while also being necessary such that its removal changes the outcome. This joint formulation yields explanations that are more balanced and complete than methods focusing on only one criterion.

Despite these advances, all three methods rely on mask-based optimization procedures, which introduce an additional step of black-box training. The explanation is produced as the outcome of optimizing continuous mask variables under task-specific loss functions, rather than through transparent, interpretable reasoning steps. Consequently, the explanation process itself remains unclear and difficult to justify to practitioners outside Computer Science, which may limit its usability in real-world decision-making settings.

\noindent\textbf{{Social Interventions.}}
Peer-based interventions leverage social connections to enable individuals within a community to influence and support one another, and have been widely used to promote positive outcomes such as HIV prevention, substance use reduction~\cite{valente2007peer}, and mental health support~\cite{davidson2006peer}. These strategies rely on trust and shared lived experiences to facilitate meaningful behavioral change.
This approach is crucial in underserved communities due to the mistrust in formal institutions like healthcare systems, social services and law enforcement. This skepticism often stems from negative past interactions within these systems. \cite{fest2013street,dworsky2009homelessness}. As a result, individuals are more likely to seek support from peers rather than from traditional providers.
To support these peer-driven methods, predictive models must go beyond classification to offer clear, interpretable insights. Providing better understanding to \textit{why} someone is at risk is vital in providing timely, actionable community-rooted interventions.

% {\color{red} Needs expansion: Give some examples of interventions -- we want to highlight that

% (1) understanding the network's role is necessary [otherwise, it is not clear why we are talking specifically about peer-based intervention and nto just regular intervention]

% (2) typical interventions are based on simple hypothesis: if X, then Y. 
% Limited by human bias on what they can see in the data or know from prior knowledge
% But population is heterogeneous ... ideal intervention may be complex per person and different for each person ... find some citations for this}

Real-world interventions designed to address population-level challenges are widely used, such as controlling the spread of infectious diseases\cite{Pastor_Satorras_2015} or mitigating misinformation diffusion\cite{vosoughi_2018_the} in social systems. These problems are inherently collective, as human behavior and decision-making are strongly influenced by interactions with others in a social network\cite{a2019_christakis,jackson_2008_social}. Therefore, when designing interventions at the population level, it is necessary to explicitly account for the underlying network structure rather than treating individuals as independent.

In practice, intervention strategies are often based on simple heuristic rules of the form ``if X, then Y,'' such as targeting individuals with certain risk factors for specific treatments or policies. However, such approaches are limited by human bias~\cite{tversky1974judgment} and partial knowledge, as they rely on what practitioners can infer from available data or prior experience\cite{ludwig2024machine}.
At the same time, real-world populations are highly heterogeneous~\cite{Athey_2016}, and the factors driving outcomes vary across individuals and their social contexts. As a result, effective intervention strategies need to be personalized and context-dependent, rather than relying on a single uniform rule applied to all individuals.

\noindent\textbf{{Explanations as Hypothesis Generators.}}
Recent work notes the role of machine learning as a tool for hypothesis generation rather than just prediction \cite{ludwig2024machine,mullainathan2017machine}. In particular, they argue that data-driven models can uncover patterns that may guide the design of interventions, especially in complex social systems where traditional theory-driven approaches are limited~\cite{lazer2009computational,lazer2020computational}. 

However, generating actionable hypotheses from model outputs remains challenging. In practice, hypothesis generation is often informal and relies heavily on human interpretation, making it difficult to systematically identify interventions that are both effective and scalable. Moreover, the space of possible interventions is combinatorial, especially in networked settings, where both individual attributes and social interactions jointly influence outcomes.
Building on these observations, a natural direction is to formalize hypothesis generation as an optimization problem. Given a limited intervention budget, one would seek to construct intervention strategies that maximize population-level impact while satisfying practical requirements such as interpretability, actionability, and computational efficiency. However, existing work has not provided a unified framework that simultaneously addresses these aspects in graph-based intervention design.
% In this work, we take a step toward formalizing hypothesis generation as an optimization problem. Given a limited intervention budget, our goal is to efficiently construct intervention strategies that maximize population-level impact while satisfying key practical requirements, including interpretability, actionability, and computational efficiency. To the best of our knowledge, prior work has not provided a unified framework that simultaneously addresses these aspects in graph-based intervention design.

\begin{figure*}[t]
    \centering
    \includegraphics[width=0.98\textwidth]{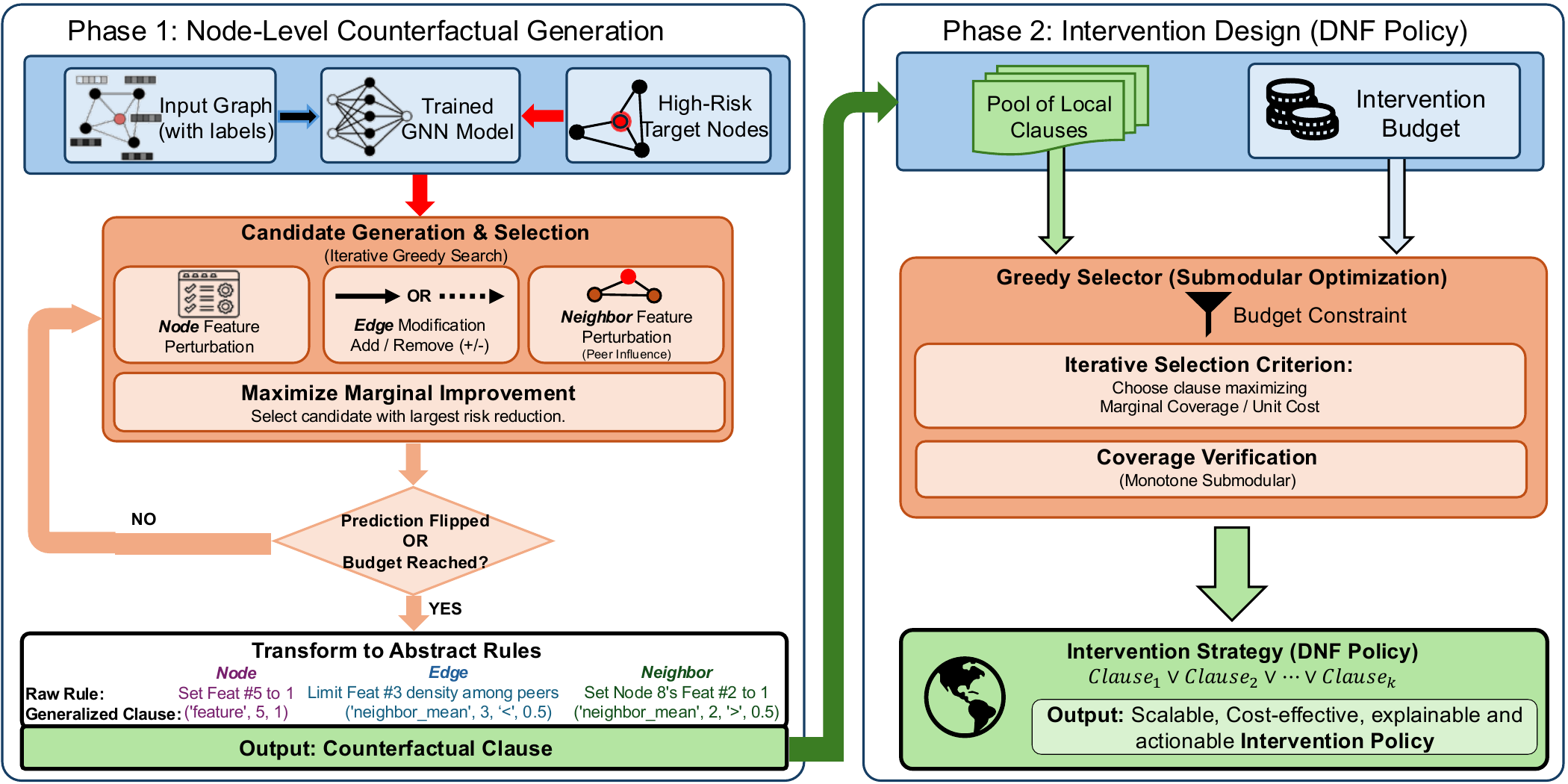}
    \caption{
            Two-phase framework: Phase 1 generates node-level counterfactual explanations via iterative greedy search over node features, edge modifications, and neighbor feature perturbations. Phase 2 aggregates local clauses into a global intervention policy using submodular optimization under a budget constraint.
    }
    \Description{Two-phase pipeline diagram. Phase 1 takes a graph and a trained GNN and runs an iterative greedy search over node features, edge modifications, and neighbor feature perturbations to produce a per-node counterfactual clause. Phase 2 aggregates these clauses into a global DNF intervention policy by greedy submodular maximization of node coverage under a budget constraint.}
    \label{fig:framework}
\end{figure*}

\section{Method and Approximation Guarantees}
\label{sec:method}

The proposed framework for intervention design on graph-structured data consists of three primary components: (1) a graph-based predictive model, (2) a node-level counterfactual search to identify node-specific interventions, and (3) a group-level optimization phase to select a robust set of intervention policies.

\subsection{Predictive Backbone}

We employ a Graph Convolutional Network\cite{GCN}(GCN) as the predictive model for node-level risk estimation. To address high-dimensional input features, we apply a fixed Principal Component Analysis (PCA)~\cite{jolliffe2016pca} projection prior to graph convolution, reducing feature dimensionality while preserving interpretability. The PCA projection matrix is pre-computed and remains non-trainable, ensuring a deterministic mapping between the original feature space and the reduced embedding space for consistent counterfactual analysis.

The GCN layers propagate information through standard neighborhood aggregation with symmetric normalization. Our framework is model-agnostic and can be readily extended to other GNN architectures such as GraphSAGE~\cite{graphsage} or GAT~\cite{GAT}.

\begin{algorithm}[t]
\caption{Node-level Counterfactual Clause Generation}
\label{alg:node_cf}
\begin{algorithmic}[1]
  \Require Graph $G$, target node $v$, GNN model $f$, maximum number of interventions $K$, edge intervention mode \texttt{edge\_mode}
  \Ensure Node-level counterfactual clause $c_v$
  
  \State Initialize $c_v \leftarrow \emptyset$, current graph $G_{\text{mod}} \leftarrow G$
  
  \While{$|c_v| < K$ \textbf{and} $f(v, G_{\text{mod}}) > 0.5$}
    \State Identify candidate interventions $I$ based on \texttt{edge\_mode}
    \State Select
    \[
      i^* \leftarrow \arg\max_{i \in I} \Big( f(v, G_{\text{mod}}) - f\big(v, G_{\text{mod}} \oplus \{i\}\big) \Big)
    \]
    \If{marginal improvement $> 0$}
      \State $c_v \leftarrow c_v \cup \{i^*\}$
      \State $G_{\text{mod}} \leftarrow G_{\text{mod}} \oplus \{i^*\}$
    \Else
      \State \textbf{break}
    \EndIf
  \EndWhile
  
  \State Convert neighbor-specific interventions in $c_v$ into threshold-based neighbor conditions
  \State \Return $c_v$
\end{algorithmic}
\end{algorithm}

\subsection{Node-level Counterfactuals}
\subsubsection{Neighborhood Feature Inclusion}

Unlike traditional counterfactual methods that focus solely on the features of the target node, our approach explicitly incorporates the graph topology into the intervention space. The scope of allowable interventions is controlled by an \texttt{edge\_mode} parameter, which defines how structural and neighborhood information is utilized:

\noindent\textbf{Feature Only (\texttt{none}).}
Interventions are restricted to perturbations of the intrinsic features of the target node $v$.

% \textbf{Edge Modification (\texttt{edges}).}
% The algorithm considers adding or removing edges between $v$ and other nodes in the graph. These topological modifications are dynamically evaluated and subsequently converted into neighborhood-level feature distribution constraints.

\noindent\textbf{Edge Modification (\texttt{edges}).} To ensure real-world feasibility, our algorithm avoids treating structural modifications as literal relationship assignments (e.g., forcing a friendship). Instead, potential edge additions or removals are dynamically evaluated and converted into actionable neighborhood-level feature distribution constraints. Specifically, whenever an edge edit significantly shifts the target node's 1-hop neighborhood feature distribution, it is abstracted into a continuous threshold rule (e.g., requiring the mean of a supportive feature $j$ among peers to exceed $\theta$). By translating deterministic edge edits into generalized feature-density constraints, the resulting interventions represent realistic public health programs and establish a rigorous, measurable cost model.

\noindent\textbf{Neighbor Feature Perturbation (\texttt{edge\_features}).}
The algorithm allows modifications to the features of the immediate neighbors of $v$, enabling the modeling of peer effects on the predicted risk of $v$. Similar to the Edge Modification approach, the rule derived from such changes focuses on changes in mean of that feature in the node's neighborhood.

\subsubsection{Generating Local Clauses}
We employ an iterative greedy search procedure (Algorithm~\ref{alg:node_cf}) to construct local counterfactual explanations. At each iteration, the algorithm selects the intervention—either a feature perturbation, an edge modification, or a neighbor feature change—that maximizes the reduction in the predicted probability for the target node.

Crucially, once a set of topological or neighborhood-level interventions is identified, these changes are transformed into \emph{Neighbor Conditions} to ensure that the resulting explanation remains both interpretable and generalizable. Specifically, when an intervention alters the neighborhood of node $v$ (e.g., through edge additions/removals or neighbor feature perturbations), we compute the induced change in the mean feature distribution of the one-hop neighborhood. This change is expressed as a threshold-based condition:
\begin{equation}
N(\text{feature}_j) \in [\text{op}, \text{threshold}],
\end{equation}
where $N(\cdot)$ denotes the mean over the one-hop neighborhood, and $\text{op} \in \{\geq, \leq\}$ specifies the direction of the constraint. 
For example, adding a peer with high mindfulness to the neighborhood of $v$ is abstracted as the mean value of a mindfulness feature within the neighborhood exceeds a predefined threshold. %By converting node-specific interventions into such neighborhood-level conditions, the resulting local policy can be meaningfully applied to and evaluated on other nodes in the graph.

\subsection{Approximation Guarantees}

% {\color{red} 
% This section is very poorly written. Please see inline comments. You should tell a clear story:

% - we used a greedy algorithm, but why would it work?

% - if the objective happens to be submodular (even approximately) then greedy provides guarantees (well-known results ...show citations)

% - if the objective is approximately additive then? We have theorem 3.1 to show greedy still provides guarantees.

% - So does the objective satisfy these properties? We do an empirical analysis

% - You must relate what you observe to the theory. 

% [If you do not undertsand this 100\%, reviewer is not going to understand or believe it]

% }

\begin{figure}[t]
    \centering
    \includegraphics[width=0.5\textwidth]{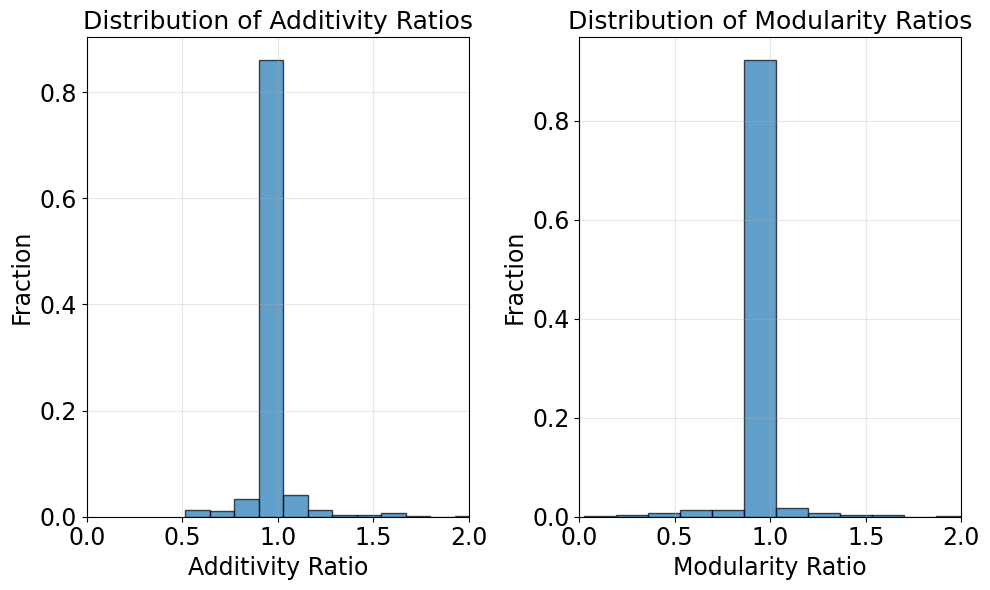}
    \caption{Empirical validation of relaxed additivity and approximate submodularity assumptions. 
    The distributions of additivity ratios (left) and modularity ratios (right) are tightly concentrated around 1, suggesting that the objective function closely satisfies near-additive and near-modularity properties, thereby supporting the theoretical guarantees.}
    \Description{Two histograms. The left panel shows the distribution of additivity ratios and the right panel the distribution of modularity ratios, measured over 5000 randomly sampled pairs of disjoint intervention sets. Both distributions are sharply peaked near a value of one, indicating that the objective is close to additive and exhibits diminishing returns.}
    \label{fig:guarantee}
\end{figure}

The node-level counterfactual search (Algorithm~\ref{alg:node_cf}) relies on a greedy strategy that selects, at each step, the single intervention yielding the largest reduction in predicted risk. A natural question arises: \emph{does greedy optimization provide any quality guarantee on the resulting intervention set?}

The answer is yes, if the objective function $f$, which measures the cumulative risk reduction achieved by a set of interventions, satisfies certain properties. For instance:

\begin{itemize}[leftmargin=*]
    \item \textbf{Submodular} ($f(A \cup \{x\}) - f(A) \geq f(B \cup \{x\}) - f(B)$ for $A \subseteq B$): Interventions exhibit \emph{diminishing marginal returns} -- the benefit of each additional intervention decreases as more interventions have already been applied. Under this condition, classical results guarantee that greedy achieves a $(1 - 1/e)$-approximation for cardinality-constrained maximization~\cite{nemhauser1978analysis}, and a $\frac{1}{2}$-approximation under knapsack constraints~\cite{sviridenko2004note}. Guarantees exist, even if $f$ is approximately submodular~\cite{das2018approximate,krause2014submodular}.
    
    \item \textbf{Approximately additive} ($f(A \cup B) \approx f(A) + f(B)$): Interventions act roughly independently. We formalize this via a relaxed additivity and prove that this also provides reasonable guarantees in Theorem~\ref{thm:greedy-bound}.
\end{itemize}

\begin{theorem}\label{thm:greedy-bound}
Let $f$ be a set function satisfying relaxed additivity with slack parameter $\alpha \in [0, 1)$:
\[
(1 - \alpha)(f(A) + f(v)) \leq f(A \cup v) \leq (1 + \alpha)(f(A) + f(v))
\]
for all sets $A$ and $v \notin A$, such that

(i) $f(\emptyset) = 0$ 

(ii) $f(A), f(v), f(A\cup v) \geq 0$

(iii) $f(A \cup v) \geq f(A)$.

Then,
the cumulative score $f_k$ obtained after $k$ greedy steps of Algorithm~\ref{alg:node_cf} satisfies:
\[
f_k \geq \gamma_k f_k^* \quad \text{where} \quad \gamma_k \geq \frac{1 - \alpha}{2(1 + \alpha)^{k - 1} - (1 + \alpha)}.
\]
\end{theorem}

\begin{proof}[Proof Sketch]
Let $f_k = f(S_k)$ be the score after $k$ greedy steps, and $f_k^*$ the optimal value. Suppose $f_k \geq \gamma_k f_k^*$. From the relaxed additivity assumption, we have:
\[
f_{k+1} \geq (1 - \alpha)(f_k + v_{k+1}^*)
\quad \text{and} \quad
f_k^* \leq (1 + \alpha)(f_k + v_{k+1}^*).
\]
Using the lower bound for $f_{k+1}$, we can write:
\[
f_k^* \leq (1 + \alpha)\left( \frac{f_{k+1}}{1 - \alpha} \right).
\]

Substitute $f_k \geq \gamma_k f_k^*$ into the inequality:
\[
f_k^* \leq f_k \left( \frac{1 + \alpha}{\gamma_k} - 1 \right) + \frac{1 + \alpha}{1 - \alpha} f_{k+1}.
\]

Assuming monotonicity ($f_{k+1} \geq f_k$) and using $f^* > f_{k+1}$, we derive a recurrence relation for the approximation factor $\gamma_k$:
\[
\frac{1}{\gamma_{k+1}} \leq (1 + \alpha) \left( \frac{1}{\gamma_k} + \frac{\alpha}{1 - \alpha} \right).
\]

Define
$
t_k = \frac{1}{\gamma_k} + \frac{1 + \alpha}{1 - \alpha}
$.
Then above inequality can be rewritten as
\[
t_{k+1} \leq (1 + \alpha) t_k \Rightarrow t_k \leq t_1 (1 + \alpha)^{k - 1}.
\]

If $\gamma_1 = 1$, then:
\[
\frac{1}{\gamma_1} = 1 \Rightarrow t_1 = 1 + \frac{1 + \alpha}{1 - \alpha} = \frac{2}{1 - \alpha}.
\]

Thus:
\[
t_k \leq \frac{2(1 + \alpha)^{k - 1}}{1 - \alpha}
\Rightarrow
\frac{1}{\gamma_k} \leq \frac{2(1 + \alpha)^{k - 1} - (1 + \alpha)}{1 - \alpha}.
\]

%Taking the reciprocal, we obtain:
\[
\implies \gamma_k \geq \frac{1 - \alpha}{2(1 + \alpha)^{k - 1} - (1 + \alpha)},
\]

which is the desired approximation bound.
\end{proof}
% \begin{proof}[Proof Sketch]
% The proof establishes a lower bound on the approximation ratio $\gamma_k$ via induction on the greedy steps:

% \textbf{Relaxed Additivity:} We utilize the slack parameter $\alpha$ to relate the $(k+1)$-th greedy score to the optimal value, yielding the inequality $f_k^* \leq \frac{1+\alpha}{1-\alpha} f_{k+1}$.

% \textbf{Recurrence Relation:} Substituting the inductive hypothesis $f_k \geq \gamma_k f_k^*$ and assuming monotonicity, we derive a linear recurrence for the reciprocal ratio: $\frac{1}{\gamma_{k+1}} \leq (1 + \alpha) \left( \frac{1}{\gamma_k} + \frac{\alpha}{1 - \alpha} \right)$.

% \textbf{Closed-form:} By defining an auxiliary sequence $t_k = \frac{1}{\gamma_k} + \frac{1+\alpha}{1-\alpha}$ and setting the initial condition $\gamma_1 = 1$, we solve the geometric progression to find the lower bound for $\gamma_k$.

% The complete proof is available in the Appendix.
% \end{proof}

While the approximation decays rapidly with $k$, we show that a small $k$ is sufficient to flip the classification. Since our objective function $f$ is implicitly defined by a trained GNN, we conduct an empirical analysis to determine if approximate submodularity and/or additivity apply.

\paragraph{Empirical Validation.}
Following the methodology of~\cite{horel2016maximization}, we randomly sample pairs of disjoint intervention sets $A$ and $B$ across target nodes and measure two diagnostic ratios:

\begin{itemize}
    \item \textbf{Additivity Ratio:} $r_{\text{add}} = f(A \cup B) \,/\, (f(A) + f(B))$. A value of near 1 indicates approximate additivity.
    
    \item \textbf{Modularity Ratio:} $r_{\text{mod}} = \sum_{x \in B}[f(A \cup \{x\}) - f(A)] \,/\, [f(A \cup B) - f(A)]$. A value of $r_{\text{mod}} \leq 1$ confirms diminishing returns (submodularity); $r_{\text{mod}} \approx 1$ indicates near-perfect modularity. A value lower than $1$ indicates approximate submodularity as defined in~\cite{das2018approximate}.
\end{itemize}

\Cref{fig:guarantee} reports the distributions of both ratios over 5000 random trials. \textbf{(i) Approximate additivity.} The additivity ratios (\Cref{fig:guarantee}, left) are tightly concentrated around 1, with the vast majority falling within $[0.9, 1.1]$. This implies that the slack parameter $\alpha$ in Theorem~\ref{thm:greedy-bound} is very small in practice. Substituting a conservative estimate of $\alpha \approx 0.1$ into the bound of Theorem~\ref{thm:greedy-bound}, even after $k = 5$ greedy steps, the approximation ratio satisfies $\gamma_5 \geq 0.49$, indicating that the greedy solution retains at least 69\% of the optimal cumulative risk reduction.\textbf{(ii) Submodularity (and near-modularity).} The modularity ratios (\Cref{fig:guarantee}, right) are also concentrated around 1, with almost all instances of $r_{\text{mod}} > 0.5$. This results in $(1 - e^{-0.5})$-approximation guarantee ~\cite{das2018approximate,krause2014submodular}.

\begin{algorithm}[t]
\caption{Greedy Global Intervention Selection}
\label{alg:global_greedy}
\begin{algorithmic}[1]
  \Require Local counterfactual candidates $\mathcal{C}$, budget $B$, coverage function $F$
  \Ensure Global intervention strategy $S$
  
  \State $S \leftarrow \emptyset$
  \State $\textit{total\_cost} \leftarrow 0$
  \State $c_{\max} \leftarrow \arg\max_{c \in \mathcal{C}} F(\{c\}) $
  \While{$\mathcal{C} \setminus S \neq \emptyset$}
    \State $c^* \leftarrow
      \arg\max_{c \in \mathcal{C} \setminus S}
      \frac{F(S \cup \{c\}) - F(S)}{|c|}$
    \If{$\textit{total\_cost} + |c^*| \le B$}
      \State $S \leftarrow S \cup \{c^*\}$
      \State $\textit{total\_cost} \leftarrow \textit{total\_cost} + |c^*|$
    \Else
      \State Remove $c^*$ from further consideration
      \State $\mathcal{C} \leftarrow \mathcal{C} \setminus \{c^*\}$
    \EndIf
  \EndWhile
  \State $S \leftarrow \arg\max_{S' \in \{S, \{c_{\max}\} \}} F(S')$
  \State \Return $S$
\end{algorithmic}
\end{algorithm}
\subsection{Graph-level Intervention}
After generating node-level counterfactuals, we wish to identify a small set of counterfactuals that can be used as a potential intervention over the population.
We construct a global intervention strategy by selecting a subset of local counterfactual clauses to form a Disjunctive Normal Form (DNF) policy that maximizes the number of high-risk nodes being covered.

\subsubsection{Problem Formulation}

Let $P \subseteq V$ denote the set of nodes predicted as positive by the base model. For each node $v \in P$, we generate a local counterfactual clause, resulting in a candidate set $\mathcal{C} = \{c_1, \dots, c_{|P|}\}$. Our objective is to select a subset $S \subseteq \mathcal{C}$ that maximizes a global coverage function $F(S)$:

\begin{equation}
\max_{S \subseteq \mathcal{C}} 
\left| \{ v \in P : \exists c \in S,\ \mathrm{Covered}(v,c) \} \right|.
\end{equation}

subject to a total cost (or complexity) constraint:
\begin{equation}
\sum_{c \in S} |c| \leq B,
\end{equation}
where $|c|$ denotes the number of individual interventions contained in clause $c$, and $B$ is a predefined budget.

\subsubsection{Coverage and Compatibility}

A clause $c$ is said to \emph{cover} a node $v$ if applying $c$ to $v$ is sufficient to flip the predicted label of $v$. In practice, we adopt two ways of covering a node $\text{Covered}(v, c)$:

\begin{enumerate}[leftmargin=*]
    \item \textbf{Model-based Flip.} For a node feature rule, we consider a node covered if applying the change to its feature will result in the trained model to flip the label, 

    \item \textbf{Rule Compatibility.} Additionally, node $v$ is still considered covered if it satisfies its own minimal counterfactual requirements $c_v$ through $c$. Specifically, this requires that $c$ contains the same self-feature interventions as $c_v$, and that the neighborhood conditions imposed by $c$ are \emph{stricter than or equal to} those required by $c_v$. For example, a clause requiring $N(X) \geq 0.8$ is compatible with a node whose local requirement is $N(X) \geq 0.5$.
\end{enumerate}

\subsubsection{Monotonicity and Submodularity}

The global coverage function $F(S)$, satisfies two key properties:
\begin{itemize}[leftmargin=*]
    \item \textbf{Monotonicity:} For any $A \subseteq B \subseteq \mathcal{C}$, we have $F(A) \leq F(B)$. Adding additional intervention clauses never decreases the total number of covered nodes.
    \item \textbf{Submodularity:} The function $F(S)$ exhibits diminishing marginal returns. For any $A \subseteq B \subseteq \mathcal{C}$ and any clause $c \in \mathcal{C} \setminus B$,
    \begin{equation}
    F(A \cup \{c\}) - F(A) \geq F(B \cup \{c\}) - F(B).
    \end{equation}
    This is because the impact of an intervention, may overlap with impact of already applied interventions diminishing its coverage. 
\end{itemize}

\subsubsection{Greedy Strategy and Approximation Guarantee}

Maximizing a monotone submodular function under a knapsack constraint is NP-hard. However, it has been shown that a simple greedy algorithm (Algorithm~\ref{alg:global_greedy}) can achieve a 0.405-approximation~\cite{tang2021revisiting}. 
It iteratively chooses the clause $c^*$ with the highest marginal coverage gain per unit cost:
\begin{equation}
c^* = \arg\max_{c \in \mathcal{C} \setminus S}
\frac{F(S \cup \{c\}) - F(S)}{|c|}.
\end{equation}
It then augments the result by finding if there is a single clause that could achieve a better coverage while being within the constraint.

% \section{Theory}
% \label{sec:theory}

% \begin{table}[t]
%   \centering
%   \caption{Summary statistics of real-world and synthetic datasets at different scales. NF and NO denote Neighbor-Features and Neighbor-Only variants.}
%   \label{tab:dataset}
%   \begin{tabular}{l|ccc}
%     \toprule
%     Dataset & \#Nodes & \#Edges & \#Feature Dim.\\
%     \midrule
%     Military & 241  & 258  & 48 \\
%     Youth  & 131 & / & 39\\
%     Syn-NF (small)    &100&150/200/400  &4/10   \\
%     Syn-NF (mid)    & 250 & 300/500/1500 &4/10 \\
%     Syn-NF (large)    &500& 800/1500/3000 & 4/10  \\
%     Syn-NO (small)    & 100 & 150/200/400 & 6\\
%     Syn-NO (mid)    & 250 &300/500/1500  &6 \\
%     Syn-NO (large)    & 500 &800/1500/3000  &6 \\
%     \bottomrule
%   \end{tabular}
% \end{table}

\begin{table}[t]
    \centering
    \caption{Summary statistics of explanation and intervention datasets. 
    The benchmark datasets (top) are used for evaluating explanation quality with ground-truth motifs, 
    while the intervention datasets (bottom) are used for assessing counterfactual-based intervention design. 
    Synthetic datasets are generated at multiple scales to control graph size and feature dimensionality. 
    NF and NO denote the Neighbor-Feature and Neighbor-Only variants, respectively.}
  \label{tab:dataset}

  \renewcommand{\arraystretch}{1.2}
  \resizebox{\linewidth}{!}{
  \begin{tabular}{lcccccc}
    \toprule
    Dataset & \#Graph & \#Ave n & \#Ave e & \#Class & \#Feat & Task \\
    \midrule
    \multicolumn{7}{c}{\textbf{Explanation Benchmark Datasets}} \\
    \midrule
    BA-Shapes   & 1    & 700   & 4100 & 4 & --   & node \\
    Tree-Cycles & 1    & 871   & 1950 & 2 & --   & node \\
    Mutag$_0$   & 2301 & 31.74 & 32.54 & 2 & 14   & graph \\
    \midrule
    \multicolumn{7}{c}{\textbf{Intervention Datasets}} \\
    \midrule
    Military & 1 & 241  & 258  & 2 & 48 & node \\
    Youth    & 1 & 131  & --   & 2 & 39 & node \\
    Syn-NF (small) & 1 & 100 & 150/200/400  & 2 & 10 & node \\
    Syn-NF (mid)   & 1 & 250 & 300/500/1500 & 2 & 10 & node \\
    Syn-NF (large) & 1 & 500 & 800/1500/3000 & 2 & 10 & node \\
    Syn-NO (small) & 1 & 100 & 150/200/400 & 2 & 6 & node \\
    Syn-NO (mid)   & 1 & 250 & 300/500/1500 & 2 & 6 & node \\
    Syn-NO (large) & 1 & 500 & 800/1500/3000 & 2 & 6 & node \\
    \bottomrule
  \end{tabular}
  }
\end{table}

\begin{table*}[t]
  \caption{Evaluation of different explanation methods against ground-truth. 
  ``Pr\%'', ``\#exp'', and ``Time'' denote Precision, explanation size, and 
  average inference time per instance (in seconds). 
  For our method, \#exp is measured by the \textbf{Minimum Information Perturbation (MIP)}, 
  defined as the smallest number of modifications required to flip the model's prediction. 
  The best and second-best results for Pr\% and Time are highlighted in \textbf{bold} and \underline{underline}.}
  \label{tab:explain_results}

  \renewcommand{\arraystretch}{1.3}
  \centering
  \resizebox{\textwidth}{!}{%
  \begin{tabular}{lccccccccc}
  \toprule
  \multirow{2}{*}{\textbf{Models}} 
      & \multicolumn{3}{c}{\textbf{BA-Shapes}} 
      & \multicolumn{3}{c}{\textbf{Tree-Cycles}} 
      & \multicolumn{3}{c}{\textbf{Mutag}$_0$} \\
  \cmidrule(lr){2-4} \cmidrule(lr){5-7} \cmidrule(lr){8-10}
   & Pr\% & \#exp & Time(s) 
   & Pr\% & \#exp & Time(s) 
   & Pr\% & \#exp & Time(s) \\
  \midrule
  GNNExplainer 
      & 60.08 & 6.0 & 0.265 
      & 68.06 & 6.0 & 0.204 
      & 59.71 & 15.0 & 0.257 \\
  CF-GNNExplainer 
      & 67.19 & 5.79 & \underline{0.232} 
      & \underline{87.40} & 3.44 & \underline{0.186} 
      & \underline{66.09} & 7.72 & \underline{0.221} \\
  CF$^2$ 
      & \underline{73.15} & 6.21 & 3.77 
      & 84.92 & 5.81 & 1.81 
      & 65.28 & 14.95 & 2.80 \\
  CF-Greedy (Ours) 
      & \textbf{78.87} & 1.20 (MIP) & \textbf{0.037} 
      & \textbf{87.50} & 1.40 (MIP) & \textbf{0.033} 
      & \textbf{92.59} & 1.19 (MIP) & \textbf{0.079} \\
  \bottomrule
  \end{tabular}%
  }
\end{table*}

\begin{table*}[t]
  \caption{Comparison of different counterfactual explanation methods on real-world and synthetic datasets. AUCC denotes the area under the coverage curve (in \%), while Cov. represents the final coverage (in \%). Time reports the average runtime in seconds. The best, second-best, and worst results for DNF-based AUCC, Coverage, and Time are highlighted in \textbf{bold}, \underline{underline}, and \textcolor{red}{red}, respectively.}
  \label{tab:main_results}

  \renewcommand{\arraystretch}{1.3}
  \centering
  \resizebox{\textwidth}{!}{%
  \begin{tabular}{c c l *{18}{c}}
  \toprule[1.1pt]
  \multirow{3}{*}{} 
      && \multirow{5}{*}{Dataset} 
      & \multicolumn{18}{c}{Method} \\
  \cmidrule(lr){4-21}

      &&  & \multicolumn{6}{c}{CF} & \multicolumn{6}{c}{CF$^2$} & \multicolumn{6}{c}{Ours-Greedy} \\
  \cmidrule(lr){4-9} \cmidrule(lr){10-15} \cmidrule(lr){16-21}

      &&  
      & \multicolumn{3}{c}{\textbf{AUCC}} 
      & \multirow{2}{*}{\raisebox{0.6\normalbaselineskip}{\vrule width0.6pt}}
      & \multirow{2}{*}{\textbf{Cov.(DNF)}} 
      & \multirow{2}{*}{\textbf{Time}}
      & \multicolumn{3}{c}{\textbf{AUCC}} 
      & \multirow{2}{*}{\raisebox{0.6\normalbaselineskip}{\vrule width0.6pt}}
      & \multirow{2}{*}{\textbf{Cov.(DNF)}} 
      & \multirow{2}{*}{\textbf{Time}}
      & \multicolumn{3}{c}{\textbf{AUCC}} 
      & \multirow{2}{*}{\raisebox{0.6\normalbaselineskip}{\vrule width0.6pt}}
      & \multirow{2}{*}{\textbf{Cov.(DNF)}} 
      & \multirow{2}{*}{\textbf{Time}} \\
  \cline{4-6} \cline{10-12} \cline{16-18}

      &&  
      & Rand. & Freq. & DNF
      &  &  &
      & Rand. & Freq. & DNF
      &  &  &
      & Rand. & Freq. & DNF
      &  &  &  \\
  \midrule

  \multirow{2}{*}{\rotatebox[origin=c]{90}{\textbf{\Large \color{darkorange}Real}}}
  & ~
  & Military 
  & 0.348 & 0.360 & \textcolor{red}{0.344} & & \textcolor{red}{37.4\%} & \underline{129.090}
  & 0.605 & 0.680 & \underline{0.912} & & \underline{97.6\%} & \textcolor{red}{348.258}
  & 0.901 & 0.933 & \textbf{0.961} & & \textbf{100.0\%} & \textbf{8.533} \\
  & ~
  & Youth 
  & / & / & / & & / & /
  & 0.273 & 0.294 & \underline{0.664} & & \underline{75.96\%} & \textcolor{red}{443.399}
  & 0.303 & 0.341 & \textbf{0.739} & & \textbf{78.88\%} & \textbf{10.448} \\
  \midrule[0.7pt]

  \multirow{12}{*}{\rotatebox[origin=c]{90}{\textbf{\Large \color{RoyalBlue}Synthetic}}}
      & \multirow{6}{*}{\rotatebox[origin=c]{90}{\textbf{Neighbor-Feature}}}
      % Updated 2026-08-21: regenerated with Rule-Compatibility coverage restored
      % (G2I repo, N=3 seeds, budget=5, eps=0.1; times on i7-6700 4-core CPU)
      & N100-E150-D10
      & 0.454 & 0.462 & \textcolor{red}{0.275} & & \textcolor{red}{33.330\%} & \underline{44.105}
      & 0.416 & 0.649 & \underline{0.811} & & \underline{97.780\%} & \textcolor{red}{130.399}
      & 0.671 & 0.858 & \textbf{0.922} & & \textbf{100.000\%} & \textbf{2.473} \\

      &
      % Updated 2026-08-21: Ours-Greedy cells regenerated post coverage fix
      % (N=3 seeds, budget=5, eps=0.1; time on i7-6700); CF/CF2 cells unchanged
      & N100-E400-D10
      & 0.346 & 0.372 & \textcolor{red}{0.543} & & \textcolor{red}{71.670\%} & \underline{76.198}
      & 0.362 & 0.422 & \underline{0.754} & & \underline{93.190\%} & \textcolor{red}{131.494}
      & 0.776 & 0.839 & \textbf{0.919} & & \textbf{100.000\%} & \textbf{4.943} \\

      &
      & N250-E300-D10
      & 0.294 & 0.340 & \textcolor{red}{0.221} & & \textcolor{red}{29.510\%} & \underline{145.480}
      & 0.297 & 0.705 & \underline{0.857} & & \underline{97.280\%} & \textcolor{red}{296.557}
      & 0.575 & 0.795 & \textbf{0.928} & & \textbf{100.000\%} & \textbf{4.525} \\

      &
      & N250-E1500-D10
      & 0.405 & 0.370 & \textcolor{red}{0.410} & & \textcolor{red}{58.810\%} & \underline{138.069}
      & 0.372 & 0.797 & \underline{0.871} & & \underline{95.940\%} & \textcolor{red}{235.938}
      & 0.690 & 0.811 & \textbf{0.915} & & \textbf{99.410\%} & \textbf{10.971} \\

      &
      & N500-E800-D10
      & 0.368 & 0.397 & \textcolor{red}{0.275} & & \textcolor{red}{38.010\%} & \underline{322.562}
      & 0.368 & 0.772 & \underline{0.896} & & \underline{97.150\%} & \textcolor{red}{560.634}
      & 0.422 & 0.842 & \textbf{0.897} & & \textbf{99.780\%} & \textbf{10.672} \\

      &
      & N500-E3000-D10
      & 0.441 & 0.286 & \textcolor{red}{0.332} & & \textcolor{red}{47.870\%} & \underline{309.519}
      & 0.336 & 0.409 & \underline{0.531} & & \underline{72.150\%} & \textcolor{red}{528.249}
      & 0.501 & 0.794 & \textbf{0.877} & & \textbf{96.030\%} & \textbf{32.099} \\

  \cmidrule(lr){2-21}
      & \multirow{6}{*}{\rotatebox[origin=c]{90}{\textbf{Neighbor-Only}}}
      % Updated 2026-08-21: regenerated with Rule-Compatibility coverage restored
      % (N=3 seeds, budget=5, eps=0.1; times on i7-6700). CF ranks second on DNF here.
      & N100-E150-D6
      & 0.822 & 0.836 & \underline{0.841} & & \underline{90.770\%} & \underline{42.691}
      & 0.554 & 0.630 & \textcolor{red}{0.685} & & \textcolor{red}{83.420\%} & \textcolor{red}{90.813}
      & 0.800 & 0.882 & \textbf{0.921} & & \textbf{100.000\%} & \textbf{0.898} \\

      &
      % Updated 2026-08-21: Ours-Greedy cells regenerated post coverage fix
      % (N=3 seeds, budget=5, eps=0.1; time on i7-6700); CF/CF2 cells unchanged
      & N100-E400-D6
      & 0.194 & 0.164 & \textcolor{red}{0.247} & & \textcolor{red}{39.730\%} & \underline{67.441}
      & 0.164 & 0.377 & \underline{0.581} & & \underline{75.740\%} & \textcolor{red}{117.229}
      & 0.764 & 0.837 & \textbf{0.880} & & \textbf{92.830\%} & \textbf{2.563} \\

      &
      & N250-E300-D6
      & 0.294 & 0.350 & \textcolor{red}{0.334} & & \textcolor{red}{43.040\%} & \underline{172.282}
      & 0.645 & 0.865 & \textbf{0.920} & & \textbf{98.920\%} & \textcolor{red}{311.179}
      & 0.385 & 0.659 & \underline{0.848} & & \underline{97.480\%} & \textbf{4.001} \\

      &
      & N250-E1500-D6
      & 0.426 & 0.283 & \textcolor{red}{0.351} & & \textcolor{red}{49.710\%} & \underline{135.757}
      & 0.302 & 0.387 & \underline{0.665} & & \underline{84.350\%} & \textcolor{red}{232.319}
      & 0.514 & 0.678 & \textbf{0.793} & & \textbf{91.220\%} & \textbf{10.594} \\

      &
      & N500-E800-D6
      & 0.309 & 0.366 & \textcolor{red}{0.301} & & \textcolor{red}{37.680\%} & \underline{291.043}
      & 0.324 & 0.863 & \underline{0.932} & & \textbf{100.000\%} & \textcolor{red}{545.142}
      & 0.900 & 0.925 & \textbf{0.951} & & \textbf{100.000\%} & \textbf{5.900} \\

      &
      & N500-E3000-D6
      & 0.562 & 0.593 & \textcolor{red}{0.470} & & \textcolor{red}{64.700\%} & \underline{265.165}
      & 0.472 & 0.703 & \underline{0.877} & & \underline{99.300\%} & \textcolor{red}{441.836}
      & 0.815 & 0.858 & \textbf{0.942} & & \textbf{100.000\%} & \textbf{13.584} \\

  \bottomrule[1.1pt]
  \end{tabular}%
  } % end resizebox
\end{table*}

%% NOTE: tables/military_clauses_all and military_clauses_constraint are
%% \input from within texts/5_experiments, exactly as in the camera-ready.
%% Do not \input them again here or they will be typeset twice.
\section{Experiments}
\label{sec:experiment}

\subsection{Datasets} 
% We use both real-world and synthetic datasets to evaluate our method due to the limited availability of large-scale, publicly accessible suicide-related data, which prevents a comprehensive evaluation using only real-world datasets. 

We use both real-world and synthetic datasets to evaluate the explanation and intervention components of our framework. 
Together, these datasets enable a comprehensive assessment of explanation accuracy and counterfactual-based intervention generation across relational and non-relational settings. 
The benchmark datasets provide ground-truth motifs for quantitative explanation evaluation, while the intervention datasets examine intervention strategies under realistic and controlled synthetic settings. 
Summary statistics are reported in~\Cref{tab:dataset}.

\subsubsection{Datasets for Explanation Evaluation}
We evaluate our explanation method on three datasets with available ground-truth motifs: BA-Shapes, Tree-Cycles, and Mutag$_0$. 
BA-Shapes and Tree-Cycles\cite{ying2019gnnexplainer} are synthetic node classification benchmarks with human-designed structural motifs (house and cycle patterns, respectively). 
Mutag$_0$ is a curated sub-dataset of Mutag\cite{debnath_1991_structureactivity,cffgnn_explainer}, constructed to isolate benzene–NO$_2$ as the sole discriminative motif for graph classification. 
These datasets allow quantitative evaluation of explanation accuracy against known ground-truth structures.

\subsubsection{Datasets for Intervention Evaluation}
We evaluate intervention strategies on two real-world suicide risk networks and a family of synthetic graphs spanning varying sizes and densities.

\textbf{Military.}
A real-world peer network constructed from ego-network interviews with active-duty military personnel, comprising 241 nodes and 258 edges. The dataset captures social support and advice-seeking relationships alongside 27 individual-level attributes spanning demographics, mental health assessments (depression, PTSD, disability), substance use, deployment history, and unit cohesion, yielding 48 features after one-hot encoding of categorical variables. Demographic attributes (e.g., race, gender, age) are designated as immutable to constrain counterfactual generation. A suicidal ideation score is computed by summing direct suicidal intention responses. While the SBQ-R manual~\cite{gutierrez2001suicidal} suggests a cutoff of 7 or higher for high suicide risk, we adopt a much more conservative definition to support early identification and intervention: individuals are labeled \emph{at-risk} if either the summed suicidal-intention score or the direct suicidal-ideation item is 2 or greater, and \emph{not at-risk} otherwise.

\textbf{Youth.}
A real-world homeless youth dataset sourced from the YouthNet-SPY study, consisting of 131 individuals with 39 features and no explicit graph structure, enabling evaluation of counterfactual interventions in a non-relational setting.
Features span three domains: self-reported attributes (demographics, employment, substance use history), neighborhood-aggregate social network proportions (e.g., peer drug use exposure, sexual health discussion), and standardized psychological assessment scores (depression, loneliness). Suicide risk is derived from self-reported ideation and planning in the past 12 months: individuals reporting either suicidal ideation or a suicide plan are labeled \emph{at-risk}, and others \emph{not at-risk}.

\textbf{Synthetic.}
We construct a family of synthetic graph datasets using a hybrid generation mechanism that combines stochastic block model (SBM)~\cite{sbm} structure with homophily-based edge formation. Each node is assigned $m$ binary attributes drawn uniformly at random, and a normalized weight vector determines the influence of each attribute on weighted-distance-based edge formation, inducing controlled homophily effects. Risk labels are assigned through a controlled scoring mechanism. In the \emph{Neighbor-Only} variant, a node's risk depends solely on neighborhood composition, defined as the proportion of neighbors possessing a designated risk-inducing attribute. In the \emph{Neighbor-Feature} variant, risk follows a diathesis--stress formulation~\cite{monroe1991diathesis}, where a node becomes high-risk only when intrinsic vulnerability (determined by specific attribute values) interacts with exposure to high-risk neighbors. Nodes are labeled \emph{at-risk} if they fall within the top 17\% of the risk score distribution.

\subsection{Experimental Setup}

\noindent\textbf{Models.} Our explanation method and intervention mechanism both rely on an underlying predictive model with reliable classification performance. Across all experiments, we use a unified three-layer Graph Convolutional Network (GCN)~\cite{GCN} as the predictive backbone, trained to perform binary node-level risk classification (\textit{at-risk} vs.\ \textit{not at-risk}) over all nodes in the graph. We adopt an 80/20 stratified train--test split and optimize the models. The trained backbone achieves average test accuracies of $\mathbf{95.85\%}$, $\mathbf{98.99\%}$, and $\mathbf{96.40\%}$ on the Military, Youth, and synthetic datasets, providing a reliable basis for the subsequent counterfactual explanations and intervention generation.
% \noindent\textbf{Models:}
% Our explanation method and intervention mechanism both rely on an underlying well-trained predictive model. Across all experiments, we use a unified three-layer Graph Convolutional Network (GCN)~\cite{GCN} as the predictive backbone.

\noindent\textbf{Evaluation Metrics.} We use the following metrics to evaluate the node-level explanations.
\begin{itemize}[leftmargin=10pt,labelwidth=0pt]

\item \textbf{Pr\% (Precision).} Precision is the proportion of predicted important edges that belong to the ground-truth important subgraph. 

\item \textbf{\#exp (Explanation Size).}  
\#exp measures the size of the counterfactual explanation. For baseline methods, it is the number of perturbed edges (or features). For our method, \#exp is defined as the \emph{Minimum Information Perturbation (MIP)}, i.e., the minimum number of modifications needed to flip the prediction.

\item \textbf{Time(s).}  
Time reports the average inference time per instance (in seconds) required to generate a counterfactual explanation.
\end{itemize} 

We use the following metrics for graph-level intervention.
\begin{itemize}[leftmargin=10pt,labelwidth=0pt]
    \item \textbf{AUCC} (Area Under the Coverage Curve).
    AUCC measures intervention efficiency by computing the area under the coverage curve, which shows the number of flipped at-risk nodes under different intervention budgets. Higher AUCC indicates broader coverage with fewer interventions.

    \item \textbf{Coverage.}
    Coverage is defined as the number of at-risk nodes whose predicted labels are successfully flipped by a given intervention strategy under a fixed budget.
    
    \item \textbf{Time(s).}
    Time reports the time required to generate counterfactual explanations and intervention strategies (in seconds).

\end{itemize}

\noindent\textbf{Hyperparameters and Implementation Details.}
Experiments were conducted using a hybrid CPU-GPU setup. Model training and inference were performed on NVIDIA RTX 2080S GPUs, while explanation generation — including perturbation and greedy orchestration — was executed on a 64-core AMD EPYC 7513 CPU with 248 GB memory. To ensure a fair comparison, all competing methods were evaluated under the same GPU and CPU settings.
For CF, the regularization coefficient $\beta$ was selected from $\{0.01, 0.1, 1\}$. 
For $CF^2$, the edge and node mask thresholds were tuned within the range $[0.2, 0.8]$. 
The intervention budget was chosen from $\{5, 10, 15, 20\}$.  The code is publicly available~\footnote{\url{https://github.com/scc-usc/G2I}}.

To evaluate the effectiveness of our proposed DNF-based intervention strategy, we compare it against two commonly used baselines:
\textbf{(1) Random}, which selects intervention candidates uniformly at random;
and \textbf{(2) Frequency}, which prioritizes candidates based on their occurrence frequency in the generated counterfactual rules. Each experiment was repeated 10 times with different random seeds, and we report the mean $\pm$ standard deviation. For clarity, only the mean values are shown in the tables.

\begin{figure}[!ht]
    \centering
    \includegraphics[width=0.5\textwidth]{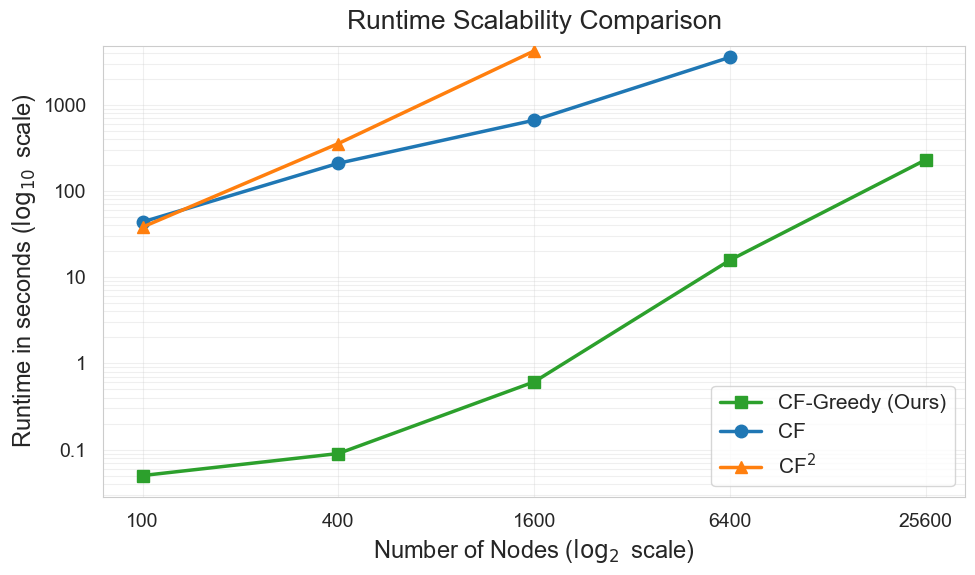}
    \caption{
            Runtime scalability of CF-Greedy, CF, and CF$^2$ on synthetic graphs with 100 to 25,600 nodes.
    }
    \Description{Line plot of runtime in seconds against the number of nodes, on logarithmic axes, for CF-Greedy, CF, and CF-squared on synthetic graphs ranging from 100 to 25,600 nodes. CF and CF-squared grow steeply and stop before the largest sizes, while CF-Greedy grows far more slowly and continues across the full range.}
    \label{fig:scalibility}
\end{figure}

% =====================================================
% TABLE 1: UNDERSTANDING (UNCONSTRAINED)
% =====================================================
% Regenerated 2026-08-22 with the Rule-Compatibility coverage criterion
% (Sec. 3.4.2) restored in compute_coverage; run: seed 42, budget 20,
% eps 0.1, 51 target nodes. Previous version backed up as
% military_clauses_all.tex.bak_20260822.

\begin{table*}[!ht]
\centering
\caption{Unconstrained CF-GREEDY results on the Military dataset (all features mutable)  -- may not be applicable for intervention design. Left: selected clauses with coverage and cost. Right: human-readable interpretation of the key risk and protective factors identified.}
\label{tab:understanding}
\begin{tabular}{p{0.4\textwidth} p{0.55\textwidth}}
\toprule
\textbf{Counterfactual Clauses} & \textbf{Interpretation} \\
\midrule
\small
\begin{enumerate}[nosep, leftmargin=*, label=\arabic*.]
    \item \texttt{race\_7} $\to$ 1 \textnormal{\small(race: other $\to$ multiracial)} \hfill \textcolor{gray}{+30}
    \item \texttt{soldier\_1} $\to$ 1 $\wedge$ \texttt{career\_intent\_1} $\to$ 1 \textnormal{\small(satisfaction: not satisfied $\to$ satisfied; career plan: other $\to$ stay until retirement)} \hfill \textcolor{gray}{+13}
    \item $\bar{x}_{\texttt{career\_intent\_3}} > 0.50$ \textnormal{\small(peer condition: $>$50\% of peers staying beyond obligation)} \hfill \textcolor{gray}{+2}
    \item \texttt{race\_5} $\to$ 1 \textnormal{\small(race: other $\to$ White)} \hfill \textcolor{gray}{+2}
    \item \texttt{race\_2} $\to$ 1 \textnormal{\small(race: other $\to$ Asian)} \hfill \textcolor{gray}{+1}
    \item \texttt{soldier\_1} $\to$ 1 \textnormal{\small(satisfaction: not satisfied $\to$ satisfied)} \hfill \textcolor{gray}{+1}
    \item \texttt{race\_3} $\to$ 1 \textnormal{\small(race: other $\to$ Black)} \hfill \textcolor{gray}{+1}
    \item \texttt{race\_2} $\to$ 1 $\wedge$ \texttt{soldier\_1} $\to$ 1 $\wedge$ $\bar{x}_{\texttt{career\_intent\_1}} > 0.33$ \textnormal{\small(race: other $\to$ Asian; satisfaction: not satisfied $\to$ satisfied; $>$33\% peers retiring)} \hfill \textcolor{gray}{+1}
\end{enumerate}
&
\small
\textbf{1. Racial identity as dominant correlate.}
Race is the single strongest predictor of the model's risk classification: changing racial category alone flips 59\% (30/51) of at-risk predictions (Clause 1), and clauses involving race (Clauses 1, 4--5, 7--8) together cover 69\% (35/51). Multiple racial categories appear across clauses, revealing that suicide risk predictions are strongly stratified along racial lines. This is not an actionable target, but may highlight known disparities across races~\cite{brenner2023trends}.

\vspace{4pt}
\textbf{2. Career commitment as a protective factor.}
Career intent---particularly long-term military commitment---combined with soldier satisfaction emerges as the second most influential factor, covering 14 additional individuals (Clauses 2, 6). These suggest that a clear sense of purpose and professional belonging are associated with lower predicted risk.

\vspace{4pt}
\textbf{3. Peer environment effects.}
Several clauses involve neighbor-mean thresholds for peers' career intent (Clauses 3, 8), indicating that the social composition of an individual's peer network influences the model's predictions beyond individual attributes.
\\
\midrule
\multicolumn{2}{l}{\small\textit{Total: 8 clauses, cost = 11, coverage = 51/51 (100\%)}} \\
\bottomrule
\end{tabular}
\end{table*}

% =====================================================
% TABLE 2: INTERVENTION (CONSTRAINED)
% =====================================================
% Regenerated 2026-08-22 with the Rule-Compatibility coverage criterion
% (Sec. 3.4.2) restored in compute_coverage; run: seed 42, budget 20,
% eps 0.1, 51 target nodes. Previous version backed up as
% military_clauses_constraint.tex.bak_20260822.
\begin{table*}[!ht]
\centering
\caption{Constrained CF-GREEDY results on the Military dataset (demographic features immutable). Left: selected clauses with coverage and cost. Right: human-readable intervention strategies derived from clauses.}
\label{tab:intervention}
\begin{tabular}{p{0.4\textwidth} p{0.55\textwidth}}
\toprule
\textbf{Counterfactual Clauses} & \textbf{Proposed Interventions} \\
\midrule
\small
\begin{enumerate}[nosep, leftmargin=*, label=\arabic*.]
    \item \texttt{career\_intent\_1} $\to$ 1 \textnormal{\small(career plan: other $\to$ stay until retirement)} \hfill \textcolor{gray}{+22}
    \item \texttt{soldier\_1} $\to$ 1 \textnormal{\small(satisfaction: not satisfied $\to$ satisfied)} \hfill \textcolor{gray}{+13}
    \item $\bar{x}_{\texttt{auditc\_1}} > 0.33$ \textnormal{\small(peer condition on peers' alcohol-use score)} \hfill \textcolor{gray}{+4}
    \item \texttt{career\_intent\_6} $\to$ 1 \textnormal{\small(career plan: other $\to$ definitely leaving)} \hfill \textcolor{gray}{+4}
    \item \texttt{soldier\_1} $\to$ 1 $\wedge$ \texttt{career\_intent\_1} $\to$ 1 \textnormal{\small(satisfaction: not satisfied $\to$ satisfied; career: other $\to$ stay until retirement)} \hfill \textcolor{gray}{+4}
    \item \texttt{career\_intent\_3} $\to$ 1 $\wedge$ \texttt{phq9\_total} $\to$ 0 \textnormal{\small(career: other $\to$ staying beyond obligation; depression: symptomatic $\to$ none)} \hfill \textcolor{gray}{+2}
    \item \texttt{soldier\_3} $\to$ 1 $\wedge$ \texttt{career\_intent\_4} $\to$ 1 \textnormal{\small(satisfaction: satisfied $\to$ dissatisfied; career: other $\to$ undecided)} \hfill \textcolor{gray}{+1}
    \item \texttt{married} $\to$ 1 $\wedge$ $\bar{x}_{\texttt{career\_intent\_1}} > 0.50$ \textnormal{\small(unmarried $\to$ married; $>$50\% peers plan to retire)} \hfill \textcolor{gray}{+1}
\end{enumerate}
&
\small
\textbf{Tier 1: Career development programs} (Clauses 1--2, 4--5; covers 84\%).
Career intent is the most actionable protective factor. Promoting long-term military commitment---through career counseling, re-enlistment decision support, and skills-to-civilian mapping---combined with soldier-satisfaction building addresses the large majority of at-risk individuals. This suggests implementing mandatory career development workshops targeting at-risk personnel.

\vspace{4pt}
\textbf{Tier 2: Clinical \& psychosocial support} (Clauses 6, 8; extends coverage to 90\%).
Reducing depression severity (PHQ-9 $\to$ 0) and relationship stability support (marriage) provide additional coverage. Recommended: stepped-care depression treatment and relationship enrichment programs (e.g., PREP for Strong Bonds).

\vspace{4pt}
\textbf{Tier 3: Network-level peer interventions} (Clauses 3, 7--8; covers 100\%).
The remaining individuals require interventions involving peer composition: embedding at-risk individuals among career-committed, stably married peers. Recommended: battle-buddy pairing with career-motivated soldiers, assignment to units with higher career commitment concentrations, and peer mentoring programs.
\\
\midrule
\multicolumn{2}{l}{\small\textit{Total: 8 clauses, cost = 12, coverage = 51/51 (100\%)}} \\
\bottomrule
\end{tabular}
\end{table*}

\subsection{Results}
\label{sec:results}
We evaluate the explanation and intervention stages separately.
\subsubsection{\textbf{Results for Explanation Evaluation}}
We follow the setup of $CF^2$~\cite{cffgnn_explainer} but use \emph{\textbf{Precision}} as the primary metric. 
Since our objective is to identify the minimal perturbation that flips a prediction, we focus on detecting truly critical components rather than maximizing coverage. 
Precision directly measures how many identified edges correspond to ground-truth important ones.

{\textbf{SOTA Performance on Explanation.}}
In \Cref{tab:explain_results}, we observe that our greedy method achieves a consistent improvement over state-of-the-art mask-training explanation methods across all datasets. Notably, the second-best performing methods are also counterfactual-based approaches, whose loss functions explicitly incorporate counterfactual objectives. This observation further supports the importance of counterfactual reasoning in identifying truly decisive graph components across both node- and graph-level classification tasks.

{\textbf{Minimum Explanation Size.}}
Unlike prior methods, which either fix explanation size ($K$ in GNNExplainer) or learn unconstrained soft masks (CF-GNNExplainer and $CF^2$), our objective directly targets the minimal perturbation necessary to change the prediction. Our minimal-perturbation principle provides a distinct perspective on explanation generation. 
Although ground-truth explanations exist in our three benchmarks, real-world settings typically lack explicit importance labels and predefined explanation sizes. 
In practice, our method yields the smallest explanation size (\#exp) across all datasets, leading to more concise and interpretable explanations with reduced computational overhead.

{\textbf{Shortest Runtime.}}
All methods were evaluated under the same hardware and experimental setup. Our framework achieves more than a one-order-of-magnitude speedup over existing methods, and in some cases approaches two orders of magnitude. Such efficiency is particularly critical in time-sensitive public decision-making settings, where rapid and reliable explanations are essential.

\subsubsection{\textbf{Results for Intervention Evaluation}}
For intervention generation, we compare three counterfactual explanation methods—CF~\cite{cfgnnexplainer}, $CF^2$~\cite{cffgnn_explainer}, and our proposed CF-Greedy—under three group-level intervention strategies in~\Cref{tab:main_results}: \emph{Random}, \emph{Frequency-based}, and \emph{DNF-based}. We evaluate intervention quality using AUCC, where higher values indicate more efficient coverage under a fixed budget, and also report Coverage and Time. For clarity, only DNF-based Coverage is shown in the main table.

{\textbf{Superior Coverage and AUCC.}}
As shown in Table~\ref{tab:main_results}, explanations generated by CF-Greedy consistently achieve substantially higher overall performance (measured by AUCC) compared to CF and $CF^2$. Notably, even when combined with the \emph{Random} and \emph{Frequency-based} intervention strategies, CF-Greedy maintains strong performance advantages.

In terms of final Coverage, CF-Greedy outperforms CF and $CF^2$ across nearly all datasets. The only exception is on the Neighbor-Only setting for N250-E300-D6, where performance is slightly lower than $CF^2$. Beyond this case, CF-Greedy consistently achieves superior coverage, with many datasets reaching 100\% coverage under the fixed budget—an outcome that CF and $CF^2$ rarely attain.

{\textbf{Superior Performance on Intervention.}}
Comparing the three intervention construction strategies across explanation methods reveals a clear structural pattern. For $CF^2$ and CF-Greedy, which generate joint feature–edge counterfactuals and thus enable heterogeneous intervention clauses, the DNF-based strategy consistently achieves the best performance. When candidate clauses contain diverse components, the DNF objective exploits combinatorial interactions and submodular-like gains for better budget allocation.

In contrast, for CF, which produces edge-only explanations, the Frequency-based strategy often performs competitively or even better. Under single-edge interventions, clauses are relatively homogeneous and highly overlapping, so selecting frequent edges already captures the dominant coverage structure, leaving limited room for combinatorial optimization. As a result, the advantage of DNF diminishes when intervention expressiveness is restricted.
Importantly, real-world intervention scenarios typically involve joint modifications of both individual attributes and relational structure. In such heterogeneous settings, the DNF framework is naturally better aligned with practical intervention design.

\textbf{Sensitivity of Mask-Based Methods.}
We tuned CF and $CF^2$ separately for each dataset, adjusting the regularization coefficient $\beta$ (CF) and mask thresholds (CF$^2$) controlling explanation sparsity and size. Optimal settings vary across datasets, indicating sensitivity to data characteristics. In contrast, our greedy method requires no dataset-specific tuning and outperforms the optimized baselines.

\textbf{Robustness to Network Scale.}
Network density is characterized by the edge-to-node ratio (E/N), reflecting average node connectivity. Real-world networks are typically sparse, with relatively small average degree compared to network size~\cite{newman2010networks}. Our synthetic graphs span low to moderately dense settings (E/N from 1 to 5), covering structures observed in social and organizational networks.

Under different density settings, we observe distinct behaviors across explanation methods. For CF, which relies on edge perturbations, performance improves as edge density increases, since denser graphs provide more opportunities for single-edge interventions. In contrast, $CF^2$ shows the opposite trend, with performance decreasing as density grows.
CF-Greedy remains stable across both sparse and dense settings. Although performance slightly declines in very dense graphs, the impact is marginal compared to CF and $CF^2$, indicating stronger robustness to structural complexity.

% \textbf{Robustness to Network Scale.}Network density can be characterized by the edge-to-node ratio (E/N), which reflects the average connectivity of nodes. Real-world networks are typically sparse, with relatively small average degree compared to network size~\cite{newman2010networks}. Our synthetic graphs span low to moderately dense settings (e.g., E/N ranging from 1 to 5), covering structural conditions commonly observed in social and organizational networks.

% Under these varying density settings, we observe distinct behaviors across explanation methods. For CF, which operates solely on edge perturbations, performance tends to improve as the number of edges increases, since denser structures provide more opportunities for effective single-edge interventions. In contrast, $CF^2$ exhibits the opposite trend, with performance decreasing as density grows.

% CF-Greedy remains stable across both sparse and dense settings. Although we observe a slight decline in very dense graphs, the impact is marginal compared to CF and $CF^2$, indicating stronger robustness to structural complexity.

%\subsubsection{Scalability}

\textbf{Shortest End-to-End Runtime.}
We measure end-to-end runtime from explanation generation to intervention construction. CF-Greedy consistently achieves the shortest runtime across all datasets, outperforming CF and $CF^2$. Although runtime grows with graph density, CF-Greedy maintains a clear efficiency advantage.

\subsubsection{\textbf{Scalability}}
We evaluate the scalability of our framework under different counterfactual explanation methods as the graph size increases from 100 to over 25,600 nodes. \Cref{fig:scalibility} reports the runtime of each strategy as a function of the number of nodes. We observe that optimization-based methods (CF\cite{cfgnnexplainer} and CF$^2$\cite{cffgnn_explainer}) become significantly slower as the graph size increases, due to repeated model evaluations during explanation generation. In contrast, the greedy strategy demonstrates significantly better scaling behavior, remaining efficient even as the graph size grows.

Beyond runtime, we also examine the practical feasibility of these strategies under fixed computational resources. We find that exhaustive and optimization-based methods (CF\cite{cfgnnexplainer} and CF$^2$\cite{cffgnn_explainer}) become infeasible beyond approximately 6,400 nodes, requiring more than 248GB of memory to execute. Under the same memory constraints, the greedy-based strategy within our framework scales to graphs with over 20,000 nodes.

%These results highlight an important trade-off within our framework: while more exhaustive strategies may provide stronger optimality guarantees, they do not scale to large graphs. In contrast, the greedy strategy offers a favorable balance between solution quality and computational efficiency, enabling explanation generation at scales relevant to real-world networked systems such as large social networks.
% {\color{red}Pick only synthetic. No need to show accuracy. Just increase the number of nodes/edges and show runtime (neighbor and feature, fix D=10, fix no of edges = 4*no. of nodes): \# nodes = 100, 400, 1600, 6400, 25600.

% Compare CF, CF2, Ours (you can skip CF and CF2 for large )
% }

\section{Discussion}

% {\color{red}Real-world intervention design

% The outputs of the greedy coverage model can be used to generate human readable hypothesis that a social scientist can understand and design studies for. We use LLMs to convert the DNF clause to a human readable hypothesis.

% Examples from our runs

% 1) When we don't restrict features to mutable ones, we generate hypotheses that help us understand the underlying process:

% \begin{quote}
% Individuals with siblings or those who have long-term military commitment (definitely staying until retirement) and siblings exhibit significantly lower suicide risk compared to those without these characteristics.    
% \end{quote}

% 2) When we restrict features that are mutable, we get hypotheses that can be used as interventions

% \begin{quote}
% Targeted interventions promoting long-term military commitment (e.g., counseling to align career intent toward definitely staying until retirement) or fostering relationship stability (e.g., relationship support for marriage) will significantly reduce suicide risk among high-risk military personnel compared to standard care.
% \end{quote}}

\subsection{Real-World Intervention Design}
\label{sec:intervention_design}

A key advantage of our counterfactual framework is that the generated DNF clauses can be directly translated into human-readable hypotheses.
We demonstrate this on the Military dataset by running CF-GREEDY under two settings:
(1)~an \emph{unconstrained} setting where all features are mutable, yielding explanatory insights about risk factors; and
(2)~a \emph{constrained} setting where demographic attributes (race, gender, sexuality, age, siblings, ACEs) are immutable, yielding actionable intervention strategies.
Both settings achieve 100\% coverage of all 51 at-risk individuals.

\textbf{Understanding Risk Factors.}
The unconstrained results (\Cref{tab:understanding}) reveal the features most predictive of the model's risk classification.
Race emerges as the single dominant factor, with a single racial category change flipping 59\% of at-risk predictions.
While race itself cannot and should not be an intervention target, this finding flags potential disparities in risk distribution that warrant further study --- for instance, whether certain racial groups face systematically higher exposure to stressors or reduced access to protective resources.
Career intent and soldier identity emerge as secondary but strong protective factors even in the unconstrained setting, suggesting these are genuine, robust correlates of lower risk rather than artifacts of demographic confounding.

\textbf{Designing Actionable Interventions.}
The constrained results (\Cref{tab:intervention}) demonstrate that, even when demographic features are locked, our framework identifies actionable interventions achieving full coverage at a modest cost increase (12 vs.\ 11).
The interventions naturally organize into a tiered structure: high-impact career programs alone cover 84\% of individuals, clinical interventions extend coverage to 90\%, and network-level strategies reach the remaining 10\%.
We note that some clauses point in opposite directions for different individuals (e.g., increased soldier satisfaction lowers predicted risk for some, while decreased satisfaction does for others), reflecting heterogeneous counterfactual pathways across subgroups rather than a single population-wide effect.
This tiered organization provides a practical deployment blueprint: programs can be prioritized by population-level impact, with career counseling as the universal first-line intervention and peer-network restructuring as a targeted complement.

Please note that none of the revealed hypotheses may be causal. The framework is meant to assist social scientists in understanding model behavior and in designing studies to test hypotheses or designing interventions.

\section{Conclusion}
\label{sec:conclusion}
We presented a theoretically grounded framework for graph explanation and intervention design based on discrete counterfactual reasoning, motivated by the goal of suicide prevention and simple explanation methods. By directly optimizing for the minimal perturbation that flips a prediction, our method reframes explanation generation as a principled discrete optimization problem rather than continuous mask approximation. For intervention design, we formulate clause selection as a DNF-based coverage objective, enabling efficient greedy approximation under budget constraints. Empirically, CF-Greedy achieves superior precision, minimal explanation size, strong robustness to network density, and substantial runtime improvements, highlighting the advantages of discrete combinatorial optimization over mask-based training approaches. By generating interpretable counterfactual explanations and translating them into actionable intervention strategies, our framework supports scalable and transparent deployment in real-world public health decision-making.

Future work includes formalizing tighter approximation guarantees for structured intervention costs and deploying selected promising interventions into the target population.

\begin{acks}
    This work was supported by the Army Research Office grant W911NF-23-1-0354.
\end{acks}

\clearpage

\bibliographystyle{ACM-Reference-Format}
\balance
\bibliography{refs}

\end{document}